\documentclass{gSCS2e}
\usepackage{epstopdf}
\usepackage{subfigure}
\usepackage{color} 
\usepackage{url} 

\theoremstyle{plain}
\newtheorem{theorem}{Theorem}[section]

\newtheorem{lemma}[theorem]{Lemma}

\theoremstyle{definition}
\newtheorem{definition}[theorem]{Definition}

\theoremstyle{remark}

\usepackage[linesnumbered,ruled,vlined]{algorithm2e} 
\usepackage{bbm} 

\begin{document}


\articletype{Article}

\title{Semi-supervised K-means++}

\author{
\name{Jordan Yoder\textsuperscript{1}$^{\ast}$\thanks{$^\ast$Corresponding author. Email: jyoder6@jhu.edu}
and Carey E. Priebe\textsuperscript{1}}
\affil{\textsuperscript{1}Applied Mathematics \& Statistics, The Johns Hopkins University, Baltimore, MD}
\received{?}
}

\maketitle

\begin{abstract}
Traditionally, practitioners initialize the {\tt k-means} algorithm with centers chosen uniformly at random.  Randomized initialization with uneven weights ({\tt k-means++}) has recently been used to improve the performance over this strategy in cost and run-time.  We consider the k-means problem with semi-supervised information, where some of the data are pre-labeled, and we seek to label the rest according to the minimum cost solution.  By extending the {\tt k-means++} algorithm and analysis to account for the labels, we derive an improved theoretical bound on expected cost and observe improved performance in simulated and real data examples.  This analysis provides theoretical justification for a roughly linear semi-supervised clustering algorithm. 
\end{abstract}


\begin{classcode} 68W20, 68W25, 68U05, 62H30 \end{classcode}

\section{Introduction}
K-means \cite{forgey1965cluster, macqueen1967some} is one of the most widely known clustering algorithms.  The basic problem it solves is as follows:  for a fixed natural number $k$ and dataset $\mathcal{X} \subset \mathbb{R}^d,$ return a set of centers $C = \{c_i \in \mathbb{R}^d :  i = 1, 2, \dots k\}$ such that it is the solution to the {\emph k-means problem}:
\begin{equation} \label{eqn:exact}
C = {\rm argmin}_{\{A \subset \mathbb{R}^d \mbox{ s.t. } |A| = k\}} \sum_{x \in \mathcal{X}} \min_{c \in A} \|x - c\|^2.
\end{equation}  Then, using this set of centers, return one of $k$ labels for each datum:
$$\ell(x) = {\rm argmin}_{i \in \{1, 2, \dots, k\}} \|x - c_i\|.$$

Lloyd's algorithm  \cite{lloyd1982least} is a particularly long-lived strategy for locally solving the k-means problem (cf. Algorithm \ref{alg:lloyd});  it suggests randomly selecting a subset of size $k$ from $\mathcal{X}$ as initial centers, then alternating updates to cluster assignments and new centers. Lloyd's algorithm does not have an approximation guarantee.  Hence, for certain datasets, it can return centers that result in a large value of the objective function in equation (\ref{eqn:exact}) with high probability.  Thus, even running independent copies  of the algorithm and choosing the best result could yield poor results.

Macqueen \cite{macqueen1967some}  conjectures that exactly solving equation (\ref{eqn:exact}) is difficult.  He is correct; Mahajan et al.\ \cite{mahajan2012planar} show that even for two dimensional data, k-means is NP-hard.  Because of this, practitioners instead seek to approximately solve the k-means problem.  Arthur and Vassilvitskii \cite{arthur2007k} present {\tt k-means++}, a randomized $\mathcal{O}(\log(k))$  approximation algorithm running in $\mathcal{O}(k n)$ time (for their initialization step, which is all that is required for the approximation bound) that works by modifying Lloyd's algorithm to choose initial centers with unequal weighting (cf. Algorithm 2).  Their results are remarkable because the algorithm runs in a practical amount of time.  This work inspired others to propose alternative randomized initializations for Lloyd's algorithm for streaming data, \citep{ailon2009streaming} parallel implementations, \citep{bahmani2012scalable} and bi-approximations with extra ($> k$) centers that can then be re-clustered to yield $k$ centers. \citep{aggarwal2009adaptive, ailon2009streaming}

In semi-supervised learning, there is additional information available about the true labels of some of the data.  These typically take the form of label information (e.g. $\ell(x_1)=2$) or pair-wise constrains (e.g. $\ell(x_1) = \ell(x_2)$ or $\ell(x_1) \not= \ell(x_2)).$  In recent years, there has been a fair amount of interest in solving problems with these additional constraints.  Wagstaff et al.\  \cite{wagstaff2001constrained} propose the {\tt COP-KMeans} algorithm, which uses a modified assignment step in Lloyd's algorithm to avoid making cluster assignments that would be in violation of the constraints.
Basu et al.\  \cite{basu2002semi} focused on using label information in their {\tt Seeded-KMeans} and {\tt Constrained-KMeans} algorithms.  Both algorithms use the centroids of the labeled points as initial starting centers.
Basu et al.\ \cite{basu2004probabilistic} use the {\tt Expectation-Maximization} (EM) algorithm \cite{dempster1977maximum} as a modified Llyod's algorithm to modify the pairwise supervision algorithms to include a step wherein the distance measure is modified (so that they do not necessarily use Euclidean distance).
Finley and Joachims \cite{finley2008supervised} learn pairwise similarities to account for semi-supervision.

The structure of the remainder of this paper is as follows.  First, we will introduce the definitions and notation to be used afterwards in the remainder of the paper.  Next, we present the main algorithm, where we modify the {\tt k-means++} algorithm for the semi-supervised with labels case.  We then prove an approximation bound that improves with the amount of supervision.  Finally, we include numerical experiments showing the efficacy of the algorithm on simulated and real data under a few performance metrics.

\begin{algorithm}
\DontPrintSemicolon 
\KwIn{$\mathcal{X}$ ($n$ datapoints)\newline$\mathcal{C}$ ($k$ initial centers)}
\KwOut{$\mathcal{C}$ (updated centers)}

\Repeat{$\mathcal{C}$ has not changed}{
	Assign each $x_i \in \mathcal{X}$ to the nearest center $c(x_i) \in \mathcal{C}$.\\
	Update each $c_j \in \mathcal{X}$ as the centroid of the points $x \in \mathcal{X}$ such that $c(x) = c_j$. 
}
\Return{$\mathcal{C}$}
\caption{Lloyd's k-means algorithm}
\label{alg:lloyd}
\end{algorithm}

\begin{algorithm}
\DontPrintSemicolon 
\KwIn{$\mathcal{X}$ (n datapoints) \newline
$k$ (number of centers)}
\KwOut{$\mathcal{C}$ (set of initial centers)}
Choose an $x \in \mathcal{X}$ uniformly at random. \\
Let $\mathcal{C} = \{x\}.$ \\
\While{card$(\mathcal{C})<k$}{
Choose a datapoint $x \in \mathcal{X}$ with probability proportional to $D^2(x)$. \\
 Update $\mathcal{C} = \mathcal{C} \cup \{x\}.$
}
\Return{$\mathcal{C}$}
\caption{Initialization of centers for k-means++}
\label{algo:kppInit}
\end{algorithm}

\section{Preliminaries}
We will present a few definitions to clarify the notation used in the theoretical results.
Recall that $\mathcal{X} \subset \mathbb{R}^d$ is our data.  We will additionally partition $\mathcal{X} = \mathcal{X}^{(u)} \cup \mathcal{X}^{(s)}$ into the unsupervised and supervised data, respectively.

\begin{definition}[Clustering]
A clustering, say $\mathcal{C}$ is a set of centers that are used for cluster assignments of unlabeled points.
\end{definition}

\begin{definition}[Potential function]
Fix a clustering $\mathcal{C}.$
Define the potential function as $\phi_{\mathcal{C}}: 2^\mathcal{X} \rightarrow \mathbb{R}^+$ be defined such that for $A \subset \mathcal{X}$,
$$\phi_{\mathcal{C}}(A) = \phi(A; \mathcal{C}) = \sum_{a \in A} \min_{c \in \mathcal{C}} \|a - c\|.$$
\end{definition}

\begin{definition}[Optimal cluster]
Let $\mathcal{C}^*$ be a clustering solving the k-means problem from equation (\ref{eqn:exact}).  Let $c^* \in \mathcal{C}^*.$  An optimal cluster $\mathcal{X}_{c^*}$ is defined such that $$\mathcal{X}_{c^*} = \{ x \in \mathcal{X} : c^* =\mbox{{\rm argmin}}_{c \in \mathcal{C}^*}\|x - c\|^2\}$$
\end{definition}

\begin{definition}[$D^2$]
Call the current clustering $\mathcal{C}$.  Define $D^2: \mathbb{R}^d \rightarrow \mathbb{R}^+$ such that $$D^2(x) = \phi(\{x\}; \mathcal{C}).$$
\end{definition}

\section{Semi-supervised K-means++ Algorithm}
We now propose an extension to the {\tt k-means++} algorithm for the semi-supervised case.  We will called this the {\tt ss-k-means++} algorithm.  Suppose that we want to partition our data $\mathcal{X} \subset \mathbb{R}^d$ into $G$ groups. Let us agree that the semi-supervision occurs in the following way: 
\begin{enumerate}
\item[(1)] choose a class $c_i$ uniformly at random;
\item[(2)] choose $g_i$ observations uniformly at random from $\mathcal{X}_{c_i}$
\item[(3)] and label these $g_i$ observations as being from class $i$. 
 \end{enumerate}
     We optionally allow repetition of steps $1-3$ to give more partially supervised classes.  The modified k-means algorithm, which is Algorithm \ref{algo:ssKppInit} followed by Algorithm \ref{alg:lloyd}, replaces the initial step of choosing a point at random by choosing $g_i$ points as above, then setting the first center to the centroid of those points.  Also, during the $D^2$ probabilistic selection process, we do not allow centers to be chosen from the supervised points.  This makes sense because that cluster is already covered.  Note that this essentially the {\tt k-means++} version of the {\tt Constrained-KMeans} algorithm \cite{basu2002semi}.

\begin{algorithm}
\DontPrintSemicolon
\KwIn{$\mathcal{X}^{(u)}$ ($n_u$ unlabeled datapoints) \newline
$\mathcal{X}^{(s)}$ ($n_s$ labeled datapoints) \newline
$L \in \{1, 2, 3, \dots, k\}^{n_s}$ (labels corresponding to the data in $\mathcal{X}^{(s)}$ \newline
$k$ (number of centers)}
\KwOut{$\mathcal{C}$ (set of initial centers)}
Let $n_\ell$ be the number of supervised datapoints with label $\ell.$ \\
Let $\mathcal{C} = \emptyset.$ \\
\For{$\ell = 1, 2, \dots k$}{
    \If{$n_\ell > 0$}{
    Let $c_\ell$ be the centroid of the labeled datapoints with label $\ell.$ \\
    Update $\mathcal{C} = \mathcal{C} \cup \{c_\ell\}.$
    }
}

\While{card$(\mathcal{C})<k$}{
Choose a datapoint $x \in \mathcal{X}^{(u)}$ with probability proportional to $D^2(x)$.\\
 Update $\mathcal{C} = \mathcal{C} \cup \{x\}.$
}
\Return{$\mathcal{C}$}
\caption{Initialization of centers for semi-supervised k-means++}
\label{algo:ssKppInit}
\end{algorithm}

\section{Theoretical Results}\label{sec:theory}
Consider the objective function $\phi(\mathcal{X}; C) = \sum_{x \in \mathcal{X}} \min_{c \in C} \|x - c\|^2,$ the potential function associated with a clustering $C$.  Arthur and Vassilvitskii \cite{arthur2007k} prove that the expectation of the potential function after the seeding phase of the {\tt k-means++} algorithm satisfies $$\mathbb{E}[ \phi(\mathcal{X})] \leq 8 (\log(k) + 2) \phi_{OPT},$$ where $\phi_{OPT}$ corresponds to the potential using the optimal $G$ centers.  We will improve this bound for our algorithm by mostly following their analysis, \emph{mutatis mutandis}.

The sketch of the proof is as follows:

\begin{enumerate}
\item Bound the expectation of the potential for the first cluster (chosen by semi-supervision)
\item Bound the expectation of the potential for clusters with centers chosen via $D^2$ weighting conditioned on the center being from an uncovered cluster.
\item Bound the expectation of the potential when choosing a number of new centers at once in a technical result
\item Specialize the technical result to our algorithm and get the overall bound
\end{enumerate}

Consider a collection of data $A$ of size $n$.  Suppose we have $g$ uniformly chosen at random members of $A$ in a set $S \subset A$ that we consider pre-labeled.  Consider the mean of these datum, say $\bar{a}_S$, to be the proposed center of $A$, then the expectation of the potential function is 
$$\mathbb{E}[\phi(A; \bar{a}_S)] = \mathbb{E}\left[ \sum_{a \in A} \|a - \bar{a}_S\|^2 \right],$$ where the expectation is over the choice of the elements of $S$.  We can compute this expectation explicitly.

\begin{lemma}
\label{BoundOnPotentialUsingRandomMean}
If $S \subset A$ is a subset of $A = \{x_i \in \mathbb{R}^d : i = 1, 2, \dots n\}$ of size $g$ chosen uniformly at random from all subsets of $A$ of size $g$, then 
$$\mathbb{E}[\phi(A; \bar{a}_S)] = \left(1+\frac{n-g}{g(n-1)}\right)\phi_{OPT}(A),$$ where $$\phi_{OPT}(A)= \sum_{a_i \in A} \|a_i - c(A)\|^2,$$ $c(A)$ is the centroid of $A$ (i.e. $c(A) = \frac{1}{n} \sum_{a \in A} a$), and $\bar{a}_S$ is the centroid of $S$.
\end{lemma}

\begin{proof}
Let $n = card(A)$. Observe 
\begin{eqnarray*}
\mathbb{E} \left[ \phi(A; \bar{a}_S ) \right] &=& \mathbb{E} \left[ \sum_{a \in A} \|a - \bar{a}_S\|^2\right] \\
&=& \mathbb{E} \left[ \sum_{a \in A} \|a -c(A)\|^2 + n \|\bar{a}_S - c(A) \|^2 \right] \mbox{ by Lemma } \ref{distanceToArbPoint}\\
&=& \sum_{a \in A} \|a -c(A)\|^2 + n \mathbb{E} \left[ \|\bar{a}_S - c(A) \|^2 \right] \\
&=& \phi_{OPT}(A) + n \mathbb{E} \left[ \|\bar{a}_S - c(A) \|^2 \right].
\end{eqnarray*}

Let us determine $ \mathbb{E} \left[ \|\bar{a}_S - c(A) \|^2 \right].$  Observe
\begin{eqnarray*}
\mathbb{E} \left[ \|\bar{a}_S - c(A) \|^2 \right]&=& \mathbb{E} \left[ \bar{a}_S^T \bar{a}_S \right] - 2 c(A)^T \mathbb{E}\left[ \bar{a}_S \right] + c(A)^Tc(A) \\
&=&\mathbb{E} \left[ \bar{a}_S^T \bar{a}_S \right]  - 2 g^{-1} c(A)^T \mathbb{E}\left[ \left(\sum_{a \in S}^n a\right) \right] + c(A)^Tc(A) \\
&=& \mathbb{E} \left[ \bar{a}_S^T \bar{a}_S \right]  - 2 g^{-1} c(A)^T \frac{g}{n} \sum_{i=1}^n a_i + c(A)^Tc(A) \mbox{ by Lemma } \ref{ExpectationOfSubset}\\
&=&\mathbb{E} \left[ \bar{a}_S^T \bar{a}_S \right]  - 2 c(A)^T c(A)+ c(A)^Tc(A) \\
&=& g^{-2}\mathbb{E} \left[ \left(\sum_{a \in S}^n a\right)^T \left(\sum_{a \in S}^n a\right) \right] -  c(A)^T c(A).
\end{eqnarray*}
Applying Lemma \ref{SecondMomentOfSubset}, we have
\begin{eqnarray*}
\mathbb{E} \left[ \|\bar{a}_S - c(A) \|^2 \right] &=& g^{-2}\left( \frac{g (g-1)}{ n (n-1)} \sum_{i \not= j} a_i^T a_j + \frac{g}{n} \sum_{i=1}^n a_i^T a_i\right) -  c(A)^T c(A)\\
&=& - \frac{n-g}{g(n-1)} c(A)^Tc(A)+ \frac{n-g}{gn(n-1)} \sum_{i=1}^n a_i^T a_i  \\
&=&  \frac{n-g}{g(n-1)} \left(\frac{1}{n}\sum_{i=1}^n a_i^T a_i  - c(A)^Tc(A)\right)  \\
&=&  \frac{n-g}{gn(n-1)} \phi_{OPT}(A).  \\
\end{eqnarray*}

Hence,  $\mathbb{E} \left[ \phi(A; \bar{a}_S ) \right] = \left(1+\frac{n-g}{g(n-1)}\right) \phi_{OPT}(A).$
\end{proof}
We present several technical lemmas used above.
\begin{lemma}
\label{distanceToArbPoint}
Let $A \subset \mathbb{R}^d$ and $c$ be the centroid of A.  Let $n:=card(A).$  For any $z \in \mathbb{R}^d$,
$\sum_{a \in A} \|a - z\|^2 = \sum_{a \in A} \|a -c\|^2 + n \| z - c \|^2$
\end{lemma}
\begin{proof}
Observe
\begin{eqnarray*}
\sum_{a \in A} \|a - c\|^2 &=& \sum_{a \in A} \| (a - z) + (z - c)\|^2 \\
&=& \sum_{a \in A} \| a - z\|^2 + 2 n(c - z)^T(z-c) + n \| z - c\|^2 \\
&=& \sum_{a \in A} \| a - z\|^2 - n \| z - c\|^2.\\
\end{eqnarray*}
Hence, $$\sum_{a \in A} \|a - z\|^2 = \sum_{a \in A} \|a -c\|^2 + n \| z - c \|^2,$$
which was what was wanted. 
\end{proof}

\begin{lemma}
\label{ExpectationOfSubset}
If $S \subset A$ is a subset of $A = \{x_i \in \mathbb{R}^d : i = 1, 2, \dots n\}$ of size $g$ chosen uniformly at random from all subsets of $A$ of size $g$, then 
$$\mathbb{E}[\sum_{x_i \in S} x_i] = \frac{g}{n} \sum_{x_i \in A} x_i.$$
\end{lemma}

\begin{proof}
Let $X_i$ be the indicator random variable that is $1$ if $x_i \in S$ and $0$ otherwise (i.e. $X_i = \mathbbm{1}_{x_i \in S}).$
Observe
\begin{eqnarray*}
\mathbb{E}[\sum_{x_i \in S} x_i] &=& \mathbb{E}[ \sum_{i=1}^{n} x_i X_i | \sum_{i=1}^n X_i = g] \\
&=&  \sum_{i=1}^{n} x_i \mathbb{E}[ X_i | \sum_{i=1}^n X_i = g] \\
\end{eqnarray*}

Observe $\mathbb{E}[ X_i | \sum_{i=1}^n X_i = g] = \frac{\binom{n - 1}{g - 1}}{\binom{n}{g}}$, the probability that $x_i$ is chosen from for a group of size $g$ from $n$ objects in $A$.
The conclusion follows. 
\end{proof}

\begin{lemma}
\label{SecondMomentOfSubset}
If $S \subset A$ is a subset of $A = \{x_i \in \mathbb{R}^d : i = 1, 2, \dots n\}$ of size $g$ chosen uniformly at random from all subsets of $A$ of size $g$, then 
$$\mathbb{E} \left[\left(\sum_{x_i \in S }x_i\right)^T \left( \sum_{x_i \in S} x_i \right)\right] = \frac{g (g-1)}{ n (n-1)} \sum_{i \not= j} x_i^T x_j + \frac{g}{n} \sum_{i=1}^n x_i^T x_i.$$
\end{lemma}

\begin{proof}
Let $X_i$ be the indicator random variable that is $1$ if $x_i \in S$ and $0$ otherwise (i.e. $X_i = \mathbbm{1}_{x_i \in S}).$
Observe
\begin{eqnarray*}
\mathbb{E}\left[\left(\sum_{x_i \in S }x_i\right)^T \left( \sum_{x_i \in S} x_i \right)\right] &=& \mathbb{E}\left[ \left(\sum_{i=1}^{n} x_i X_i \right)^T \left(\sum_{j=1}^{n} x_j X_j \right) | \sum_{\ell=1}^n X_\ell = g\right] \\
&=&  \mathbb{E}\left[ \sum_{i,j}^{n} x_i^T x_j X_i X_j \ | \sum_{\ell=1}^n X_\ell = g\right] \\
&=&  \sum_{i,j}^{n} x_i^T x_j \mathbb{E}\left[  X_i X_j \ | \sum_{\ell=1}^n X_\ell = g \right]. \\
\end{eqnarray*}

Observe $$ \mathbb{E}\left[  X_i X_j \ | \sum_{\ell=1}^n X_\ell= g\right]  =  \begin{cases}
\frac{\binom{n-2}{g-2}}{\binom{n}{g}} & \mbox{ if } i \not=j \\
\frac{\binom{n - 1}{g - 1}}{\binom{n}{g}} & \mbox{ if } i = j \end{cases},$$ since the first case represents the probability that $x_i$ and $x_j$ are chosen together and the second case represents the probability that $x_i$ is chosen, as $X_i^2 = X_i$.
The conclusion follows. 
\end{proof}

The first result will handle the semi-supervised cluster.  Suppose that $C_{OPT}$ is the optimal set of cluster centers.  Now, we consider the contribution to the potential function of a cluster $\mathcal{X}^*:=\{x \in \mathcal{X} : c^* = argmin_{c \in C_{OPT}} \|x - c\|^2\}$ from $C_{OPT}$ when a center is chosen from $C^*$ with $D^2$ weighting. If we can prove a good approximation bound, then we can say that conditioned on choosing centers from uncovered clusters, we will have a good result on average.

\begin{lemma}
\label{res:dsqContrib}
Let $C$ be the current (arbitrary) set of cluster centers.  Note that $C$ is probably not a subset of $C_{OPT}$.  Let $\mathcal{X}^*$ be any cluster in $C_{OPT}$.  Let $x^*$ be a point from $\mathcal{X}^*$ chosen at random with $D^2$ weighting.   Then, $$\mathbb{E}\left[ \phi(\mathcal{X}^*; C \cup \{x^*\}) | x^* \in \mathcal{X}^* \right] \leq 8 \phi(\mathcal{X}^*; C_{OPT}),$$ where the expectation is over the choice of new center $x^*$.
\end{lemma}
\begin{proof}
Unchanged from Lemma 3.2 in \cite{arthur2007k}.
\end{proof}

\begin{lemma}
\label{lem:inductiveBound}
Fix a clustering $C.$  Suppose there are $u \in \mathbb{N}$ uncovered clusters from the optimal clustering $C_{OPT}.$  Denote the points in these uncovered clusters as $\mathcal{X}_u$ (not to be confused with $\mathcal{X}^{(u)}$).  Let $$\mathcal{X}_c = \mathcal{X} - \mathcal{X}_u$$ be the set of points in covered clusters. Use $D^2$ weighting (excluding supervised data) to add $t \leq u$ new centers to $C$ to form $C'$.  In a slight abuse of notation, let $\phi(\cdot) = \phi(\cdot; C)$, $\phi'(\cdot) = \phi(\cdot; C'),$ and $\phi_{OPT} = \phi(\cdot; C_{OPT})$.  Then, 
$$\mathbb{E}[\phi'(\mathcal{X})] \leq \left(\phi(\mathcal{X}_c) + 8 \phi_{OPT}(\mathcal{X}_u)\right)(1+H_t) + \frac{u-t}{u}\phi(\mathcal{X}_u),$$ where $H_t = \sum_{j=1}^t \frac{1}{j}.$
\end{lemma}
\begin{proof}

We have that the probability of choosing a point from a fixed set $A$ with $D^2$ weighting ignoring supervised points is $$\frac{\phi(A \cap \mathcal{X}^{(u)})}{\phi(\mathcal{X}^{(u)})}.$$  Further, note that
$\mathcal{X}^{(s)} \cap \mathcal{X}_u = \emptyset$, since all supervised clusters are covered.

Following the argument in \cite{arthur2007k} using the above probabilities, we have our result.
\end{proof}

\begin{theorem}
\label{thm:apxbnd}
Suppose our story about how the supervision occurs holds.  Let $C = \emptyset.$  For each label $\ell$ that we have supervised exemplars of, add the centroid of the supervised data labeled $\ell,$ say $c_\ell$ to C.  Suppose that $card(C) = G.$  Let $n_{\ell_j}$ be the number of supervised exemplars with label $\ell_j$ for $j=1, 2, \dots, G.$  
Then, we have $u=k-G$ uncovered clusters.  
Add $t=u$ new centers using $D^2$ weighting ignoring the supervised points.  The expectation of the resulting potential, $\phi'$, is then
$$\mathbb{E}[\phi'(\mathcal{X})] \leq 8 \phi_{OPT}(\mathcal{X})(2+\log(k-G)).$$

\end{theorem}
\begin{proof}
Applying Lemma \ref{lem:inductiveBound} with $u=t=k-G,$ we have
$$\mathbb{E}[\phi'(\mathcal{X})] \leq \left(\sum_{j=1}^G\left(\phi(\mathcal{X}_{\ell_j}) -8 \phi_{OPT}(\mathcal{X}_{\ell_j})\right)+ 8 \phi_{OPT}(\mathcal{X})\right)(1+H_{k-G}).$$
Applying Lemma \ref{BoundOnPotentialUsingRandomMean} to each $\phi(\mathcal{X}_{\ell_j}),$ we have
$$\mathbb{E}[\phi'(\mathcal{X})] \leq \left(\sum_{j=1}^G \left(7 - \frac{n-n_{\ell_j}}{n_{\ell_j}(n-1)}\right) \phi_{OPT}(\mathcal{X}_{\ell_j}) + 8 \phi_{OPT}(\mathcal{X})\right)(1+H_{k-G}).$$
Finally, using the fact that $H_{k-G} \leq 1 + \log(k-G)$, we have our result.
\end{proof}
The end result is a modest improvement over that of \cite{arthur2007k} that scales with the level of supervision. The final inequality in the proof is tighter than the result stated in the theorem, since the factor of $8$ could be lower depending on the contributions of the supervised clusters in the optimal clustering.

\section{Numerical Experiments}
\subsection{Performance Measures}
We use several measures for each experiment.  First, we use the cost, as estimated by the potential function.  For comparing to the theoretical bound, we also use the fraction of optimal cost, where ``optimal'' is derived by taking the centroids for each class as determined by the ground truth labels.  Next, we use the number of Lloyds iterations until convergence.

Finally, we will use the Adjusted Rand Index (ARI) \cite{hubert1985comparing}, which is an index that compares how closely two partitions agree.  The ARI is the Rand index, the ratio of number of agreements between two partitions, after adjusting for chance.  It is essentially chance at 0, meaningless $<0$, and perfect at its maximal value, unity.  Since ground truth labels are available for our datasets, we can compare them to the partitions yielded from the output of the algorithms in Section \ref{sec:alg}.  Thus, a large ARI value indicates good clustering performance as determined by fidelity to the ground truth partition.

\subsection{Data}\label{sec:sskppData}
We showcase our algorithm on three datasets (cf. Figure \ref{fig:data} for depiction).  The first, {\tt Gaussian Mixture} was inspired by both \cite{arthur2007k, bahmani2012scalable}.   We drew $k=24$ centers from the 15-dimensional hypercube with side length of $10$.  For each center $c_\ell$, we drew $n_\ell=100$ points from a multivariate Guassian with mean $c_\ell$ and identity covariance.  This dataset is remarkable because it is easy to cluster by inspection (at least with larger side-length, as in the original papers) yet is difficult for Lloyd's algorithm when initialized with bad centers.  For our chosen side length, it is not easy to cluster by eye.  Note that the supervision story (where centroids of the class labels correspond to best centers) is likely to hold for most realizations of the data.

The next two datasets are real data for which the assumption that the labels match up with minimum cost clusters is not met.
The second dataset is the venerable {\tt Iris} dataset \citep{fisher1936use}, which uses $d=4$ variables to describe $k=3$ different classes of flowers.  While this dataset is old, it is nonetheless difficult for {\tt k-means} to handle from a clustering standpoint.  This fact is widely known; indeed, even the Wikipedia page for the {\tt Iris} dataset has a rendering of {\tt k-means} failing on it.\citep{wiki:Iris}  We compared the ARI for this dataset and the {\tt Gaussian Mixture} dataset while varying the ratio of side length of the hypercube to standard deviation ($\Sigma = \sigma I$ with $\sigma = 1$ fixed), and we found that the datasets were roughly equivalent for side length around 3.25.  This is under one percent of the side length and $10^{-32}$ times the volume of the {\tt norm25} dataset \cite{bahmani2012scalable} that our {\tt Gaussian Mixture} dataset is based on.  Thus, we observe that the {\tt Iris} dataset is harder to cluster than the synthetic dataset.

The third dataset, {\tt Hyperspectral}, is a Naval Space Command HyMap hyperspectral dataset representing an aerial photo a runway and background terrain of the airport at Dalgren, Virginia as originally seen in \cite{priebe2004integrated}  (cf. Figure \ref{fig:hypPic} for a depiction of the location). Each pixel in the photo is associated with $d=126$ features representing different spectral bands (e.g. visible and infrared).  We took the first six principal components to form a dataset with $n =14748$ data in $\mathbb{R}^6$, as chosen by the minimum number of dimensions to capture $> 97.5\%$ of the total variance.  The first two principal components are depicted in Fig \ref{fig:data}.  The $k=7$ classes are the identities of each pixel (i.e.  runway, pine, oak, grass, water, scrub, and swamp).  Based on the ARI scores presented in the forthcoming results section, this dataset is only a little easier to cluster than {\tt Iris}.

\begin{figure*}[t!]
    \centering
    \includegraphics[width=.95\textwidth]{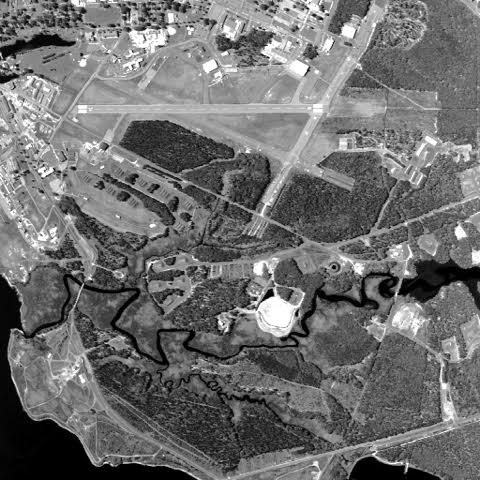}
    \caption{Image corresponding to the {\tt Hyperspectral} dataset as seen in Figure 2 of \cite{priebe2004integrated}.  Each pixel can be classified according to what it represents.}
    \label{fig:hypPic}
\end{figure*}

\begin{figure*}[t!]
    \centering
    \subfigure[{\tt Gaussian Mixture}] {\includegraphics[width=.45\textwidth]{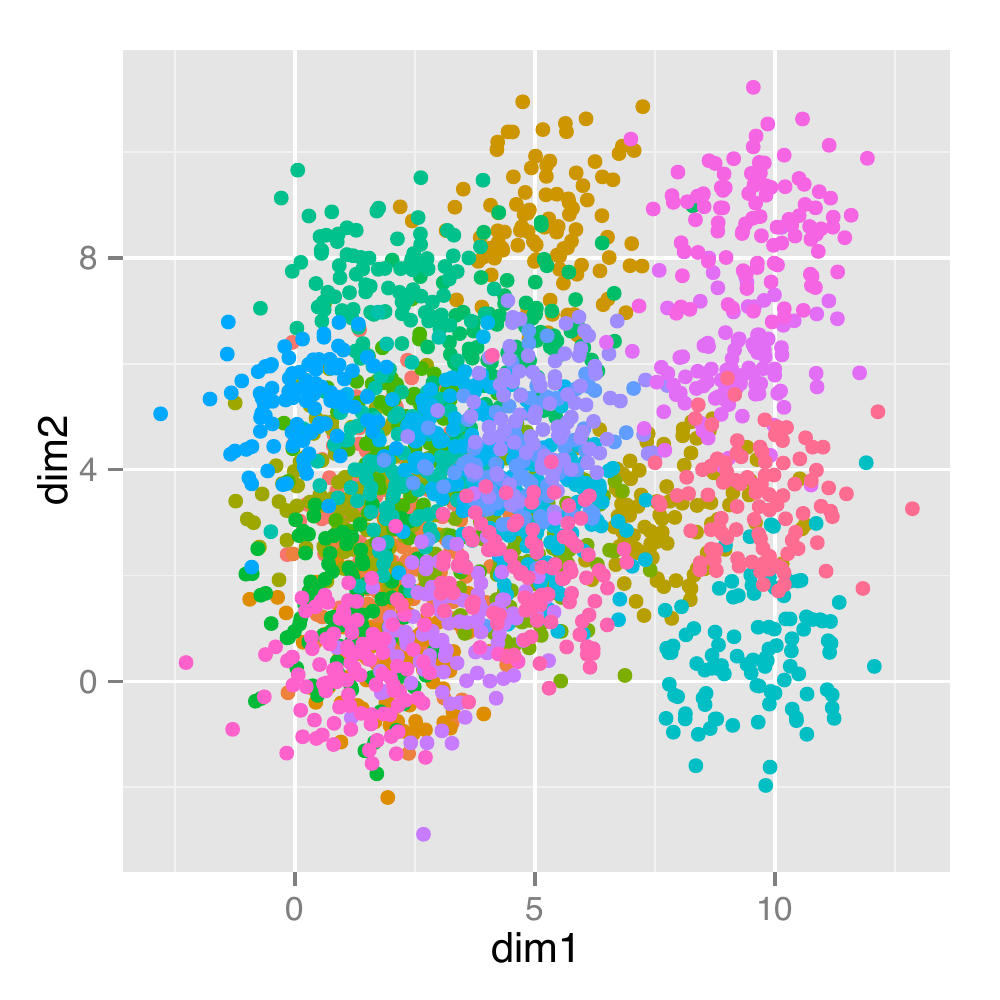}}
     \subfigure[{\tt Iris}]{ \includegraphics[width=.45\textwidth]{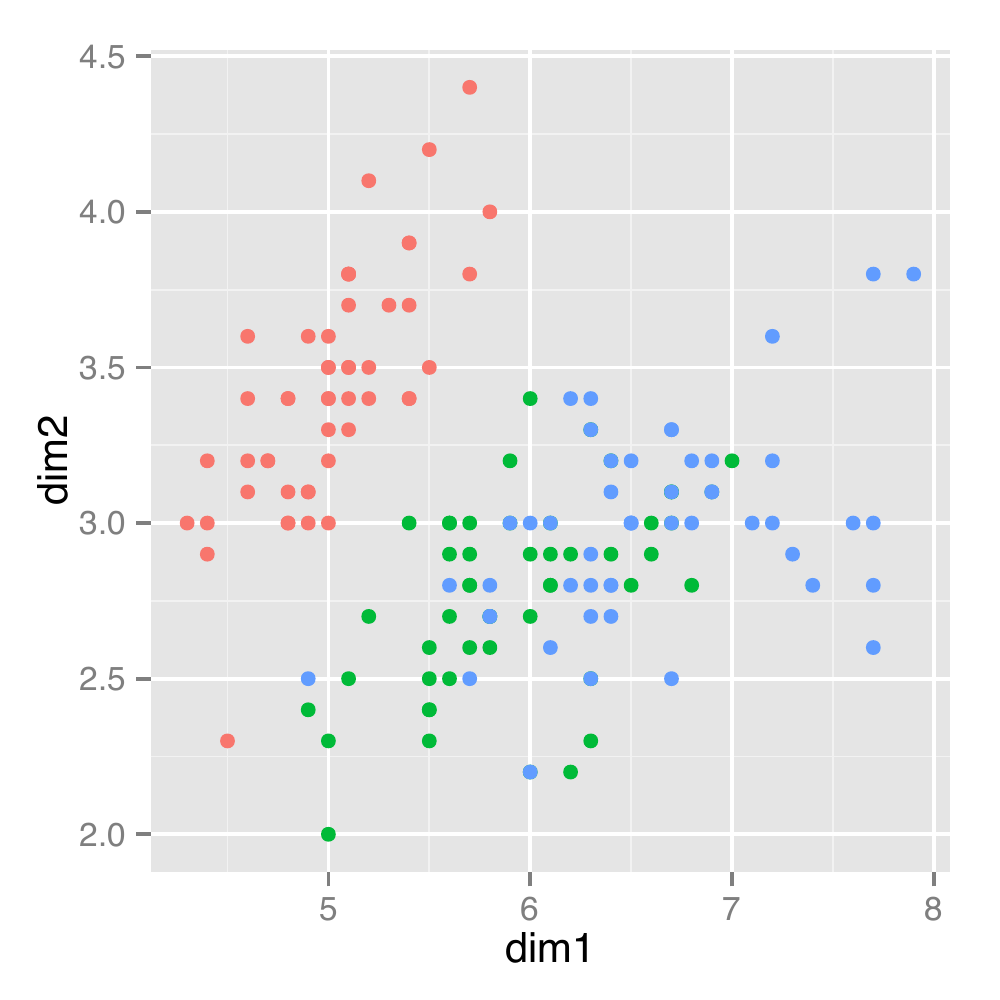}}
      \subfigure[{\tt Hyperspectral}]{ \includegraphics[width=.45\textwidth]{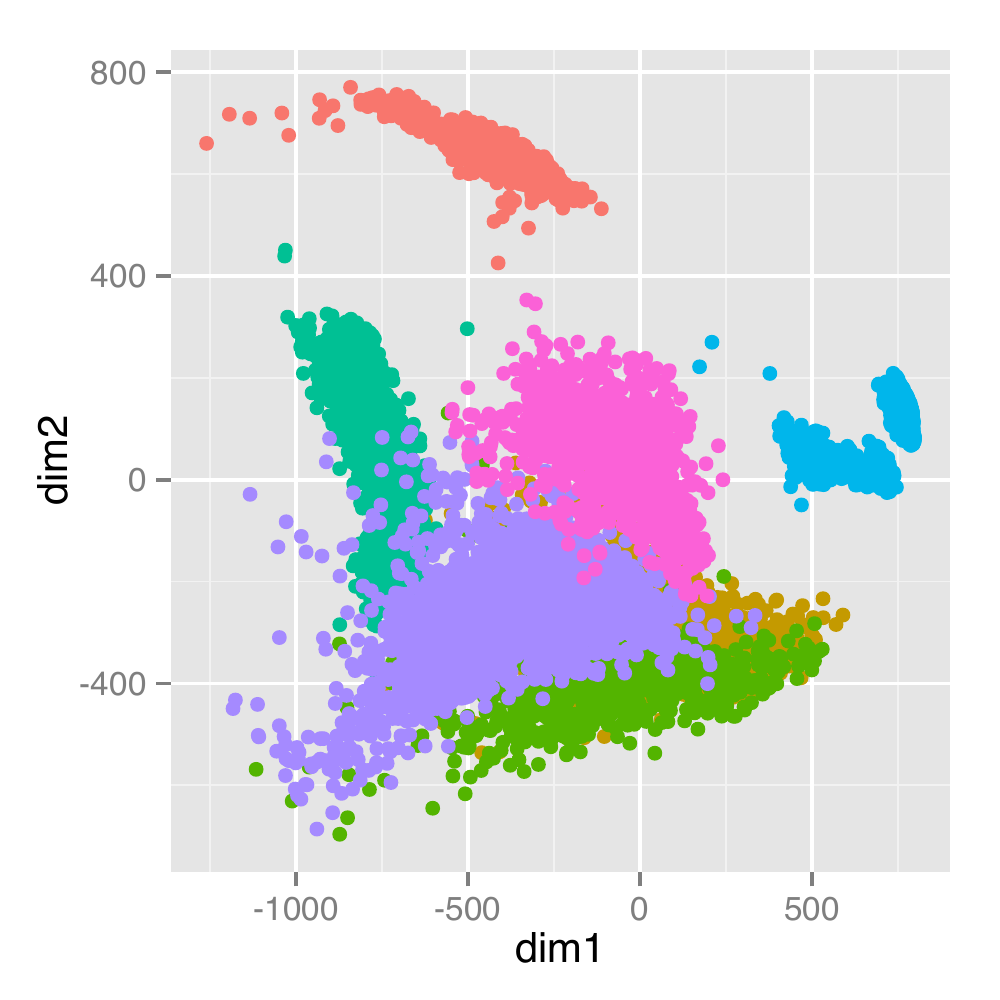}}
    \caption{First two dimensions of the datasets (one realization for {\tt Gaussian Mixture}).  Because {\tt Gaussian Mixture} has 13 more dimensions than are shown here, clustering it is considerably easier than this figure would imply.  Note, however, that we have overlapping classes (as denoted by the colors) in all datasets.}
    \label{fig:data}
\end{figure*}

\subsection{Algorithms}\label{sec:alg}
For both datasets, we apply several algorithms: {\tt ss-k-means++}, {\tt Constrained-KMeans},  {\tt ss-k-means++} (without Lloyds), {\tt Constrained-KMeans} (without Lloyds), and  {\tt Constrained-KMeans} algorithm initialized at the true class centroids.  We consider the latter algorithm as an approximation to the optimal solution.  {\tt Constrained-KMeans} and {\tt Constrained-KMeans} (without Lloyds) use a random sample of the unsupervised data weighted uniformly for the remaining initial centers (after using centroids of the labeled points).  The algorithms without Lloyds use their respective initialization strategy to choose initial centers then move straight to class assignment without updating the initial centers.  We consider these algorithms as ``initialization only'' methods for this reason.

\subsection{Results}
We vary the supervision level from 0\% to 100\%, where we add supervised classes and sample 5/5/50 datapoints per class to label for {\tt Guassian Mixture}, {\tt Iris}, and {\tt Hyperspectral}, respectively. Note that this is percent of clusters which have exemplars and not percent of all points which are labels. Also, at 100\% supervision, {\tt ss-k-means++} and {\tt Constrained-KMeans} are the same, since there are no additional centers to choose.  We did not allow the supervised data to change cluster assignment, so the approximation to the optimal can change with the level of supervision and with different supervised data chosen                                                                                                                                                                                     .  We set $k$ equal to the true number of groups (24 for {\tt Gaussian Mixture}, 3 for {\tt Iris}, and 7 for {\tt Hyperspectral}).  We used 100 Monte Carlo replicates at each level of supervision.

Figure \ref{fig:cost} shows the cost as the level of supervision changes.  We observe the cost decreases with more supervision.  Also, we see the same relative performances of the algorithms, with the ++ version outperforming the benchmark.  Observe that the approximation to the optimal solution is the best.  Figure \ref{fig:fcost} depicts the theoretical bound.  All algorithms are below the bound (in expectation).

\begin{figure*}[t!]
    \centering
    \subfigure[{\tt Gaussian Mixture}] {\includegraphics[width=.45\textwidth]{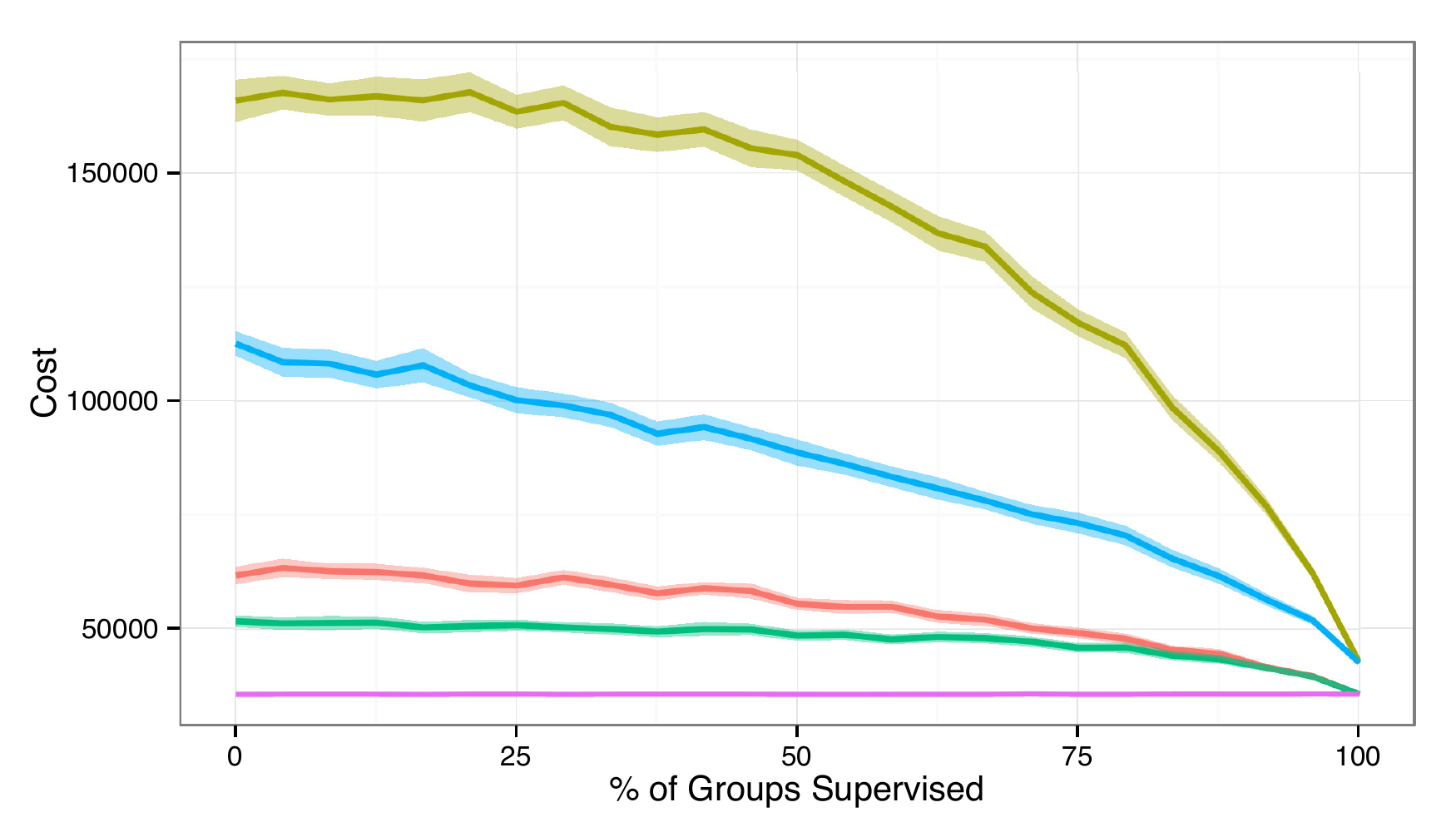}}
     \subfigure[{\tt Iris}]{ \includegraphics[width=.45\textwidth]{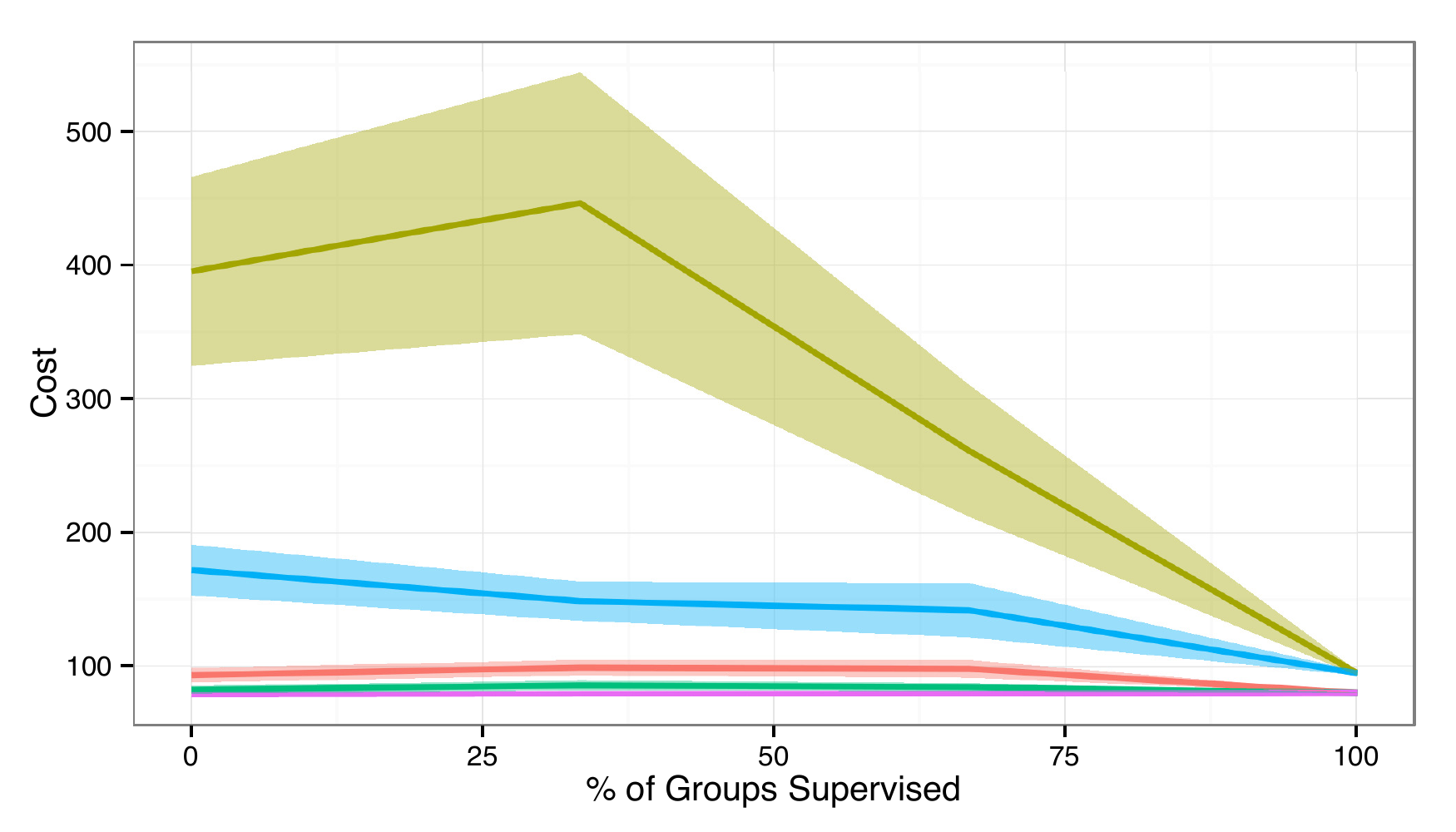}}
      \subfigure[{\tt Hyperspectral}]{ \includegraphics[width=.45\textwidth]{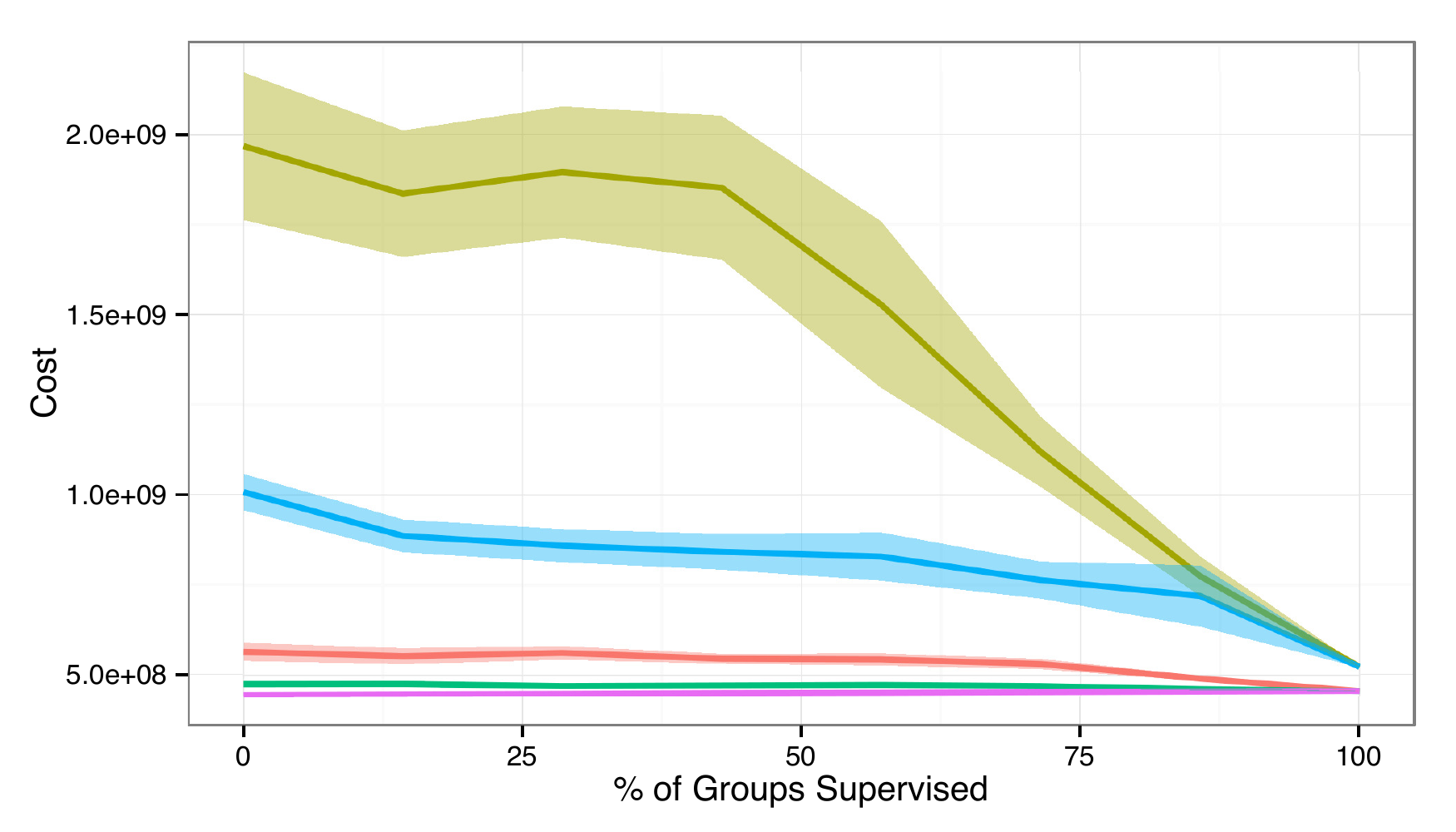}}
    \caption{Cost (value of the potential) shown as a function of the level of supervision for 100 Monte Carlo replicates.  Shading indicates $\pm$ two standard deviations.
         Colors indicate algorithm:\\
    \textcolor[rgb]{.62, .65, 	.03}{gold}: {\tt Constrained-KMeans} (without Lloyds iterations);\\
    \textcolor[rgb]{0.0, .72, .96}{blue}:  {\tt ss-k-means++} (without Lloyds iterations);\\
    \textcolor[rgb]{1.0, .44, .42}{red}:  {\tt Constrained-KMeans};\\
    \textcolor[rgb]{0.0, .76, .53}{green}: {\tt ss-k-means++}; and\\
    \textcolor[rgb]{.91, .40, .94}{pink}: {\tt Constrained-KMeans} initialized at true centroids of labels.}
    \label{fig:cost}
\end{figure*}

\begin{figure*}[t!]
    \centering
    \subfigure[{\tt Gaussian Mixture}] {\includegraphics[width=.45\textwidth]{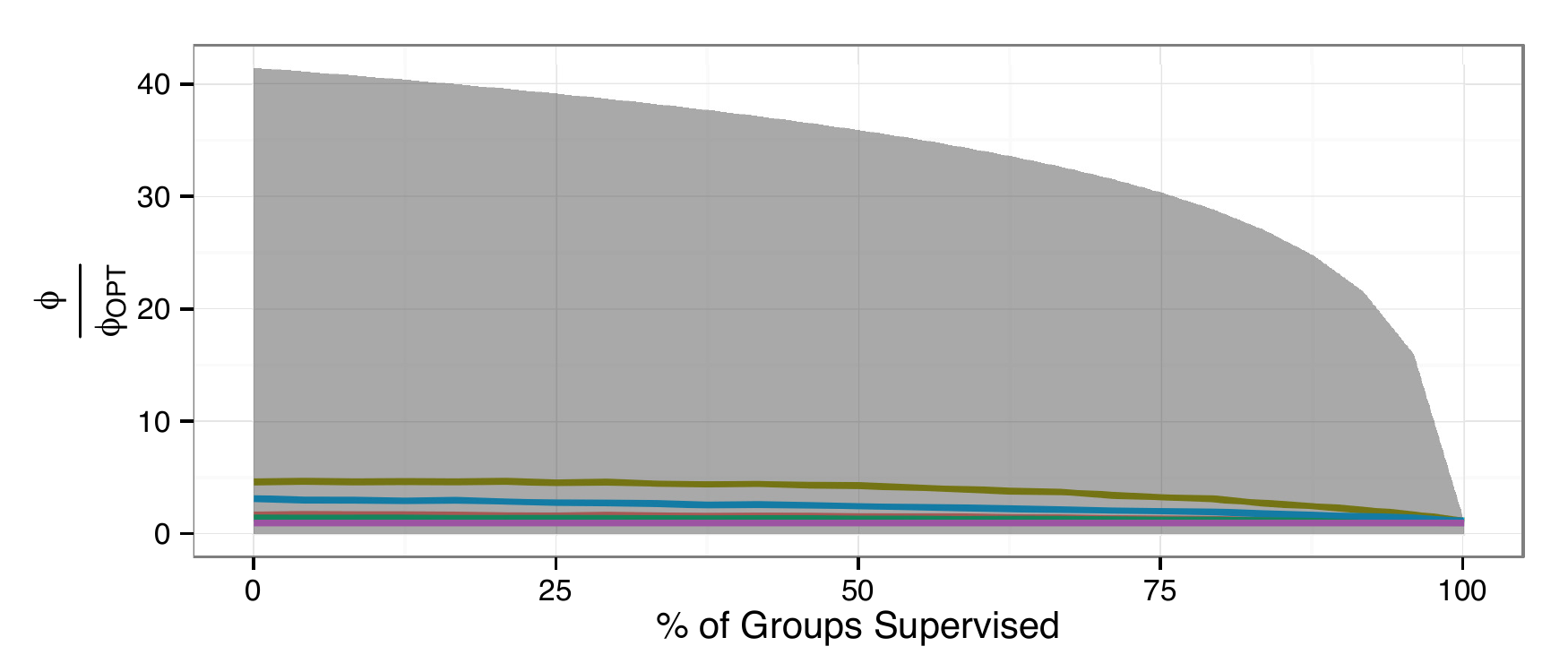}}
     \subfigure[{\tt Iris}]{ \includegraphics[width=.45\textwidth]{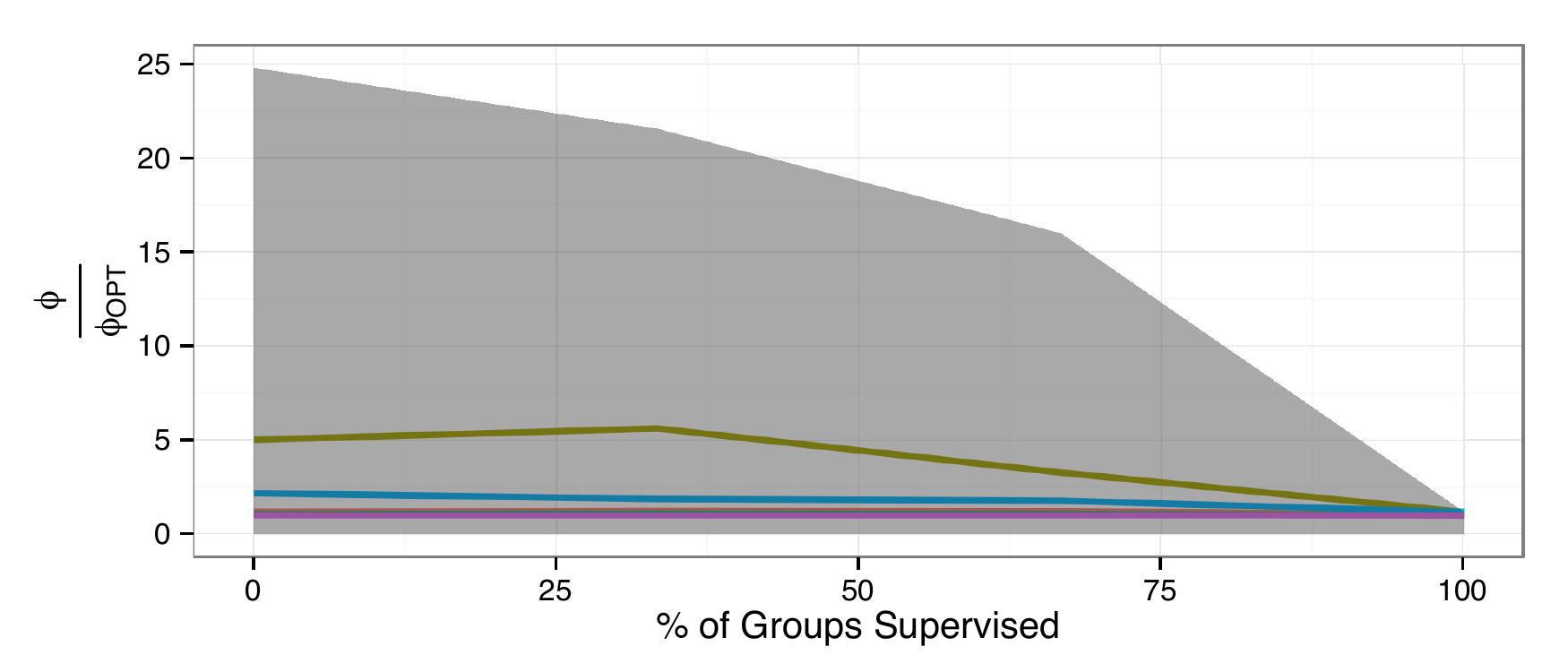}}
\subfigure[{\tt Hyperspectral}]{ \includegraphics[width=.45\textwidth]{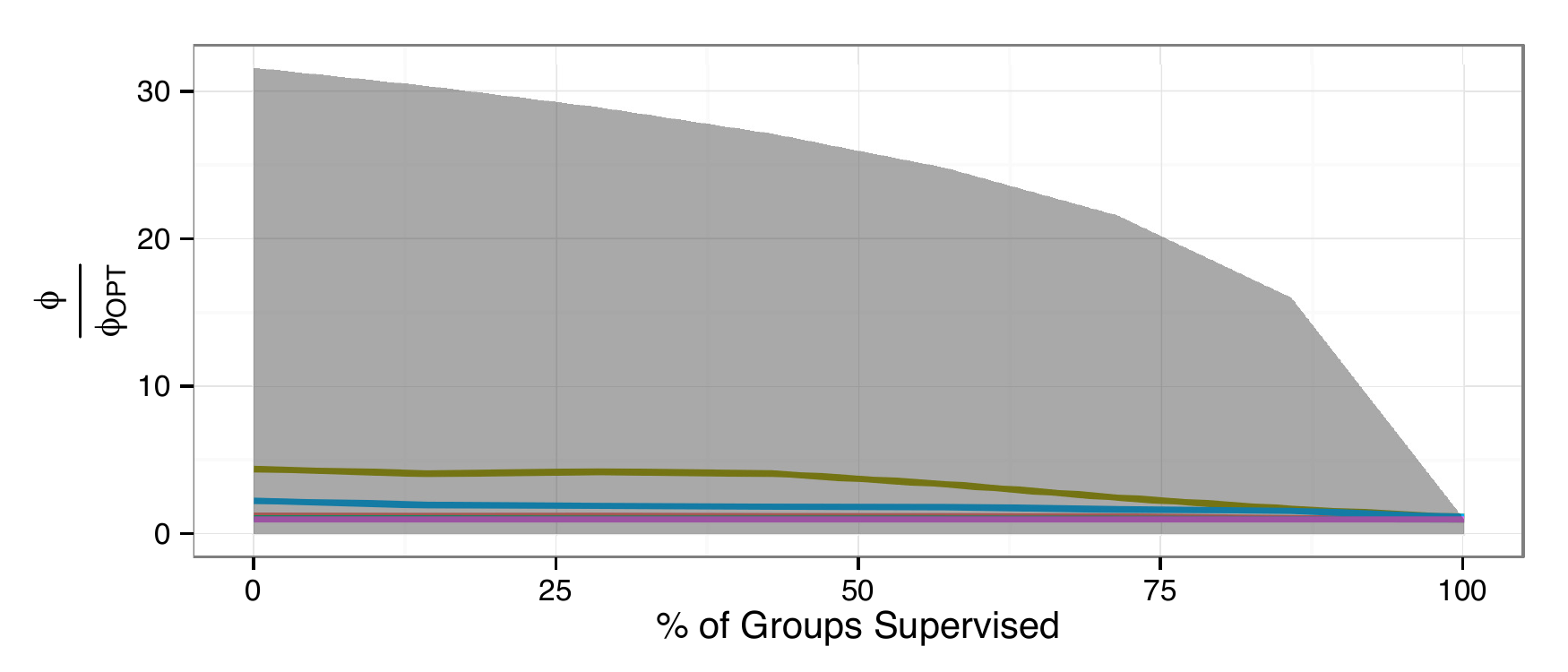}}
  \caption{Fractional cost (value of the potential over an estimate of the optimal) plotted as a function of the level of supervision for 100 Monte Carlo replicates.  Shading around the lines indicates $\pm$ two standard deviations.  The shaded region is the region corresponding to the theoretical cost in expectation from Section \ref{sec:theory}.   Colors indicate algorithm:\\
    \textcolor[rgb]{.62, .65, 	.03}{gold}: {\tt Constrained-KMeans} (without Lloyds iterations);\\
    \textcolor[rgb]{0.0, .72, .96}{blue}:  {\tt ss-k-means++} (without Lloyds iterations);\\
    \textcolor[rgb]{1.0, .44, .42}{red}:  {\tt Constrained-KMeans};\\
    \textcolor[rgb]{0.0, .76, .53}{green}: {\tt ss-k-means++}; and\\
    \textcolor[rgb]{.91, .40, .94}{pink}: {\tt Constrained-KMeans} initialized at true centroids of labels.}
    \label{fig:fcost}
\end{figure*}

Figure \ref{fig:iter} shows the number of iterations before Lloyd's converges.  We can see that improved selection of by $D^2$ weighted randomization leads to fewer iterations before convergence.  We expected this;  Arthur and Vassilvitskii \cite{arthur2007k} observed a similar phenomenon with no supervision.  More supervision did not seem to affect the number of iterations until very high levels (near 100\%).  For the real world datasets, we can see that the approximation to the optimal algorithm required more than one iteration to converge, indicating that the centroids of the true class labels do not match with the locally minimal cost solutions.  This means that the conditions for the supervision in our proofs do not hold for this dataset.  Nevertheless, both cost and ARI improve with additional supervision.

\begin{figure*}[t!]
    \centering
    \subfigure[{\tt Gaussian Mixture}] {\includegraphics[width=.45\textwidth]{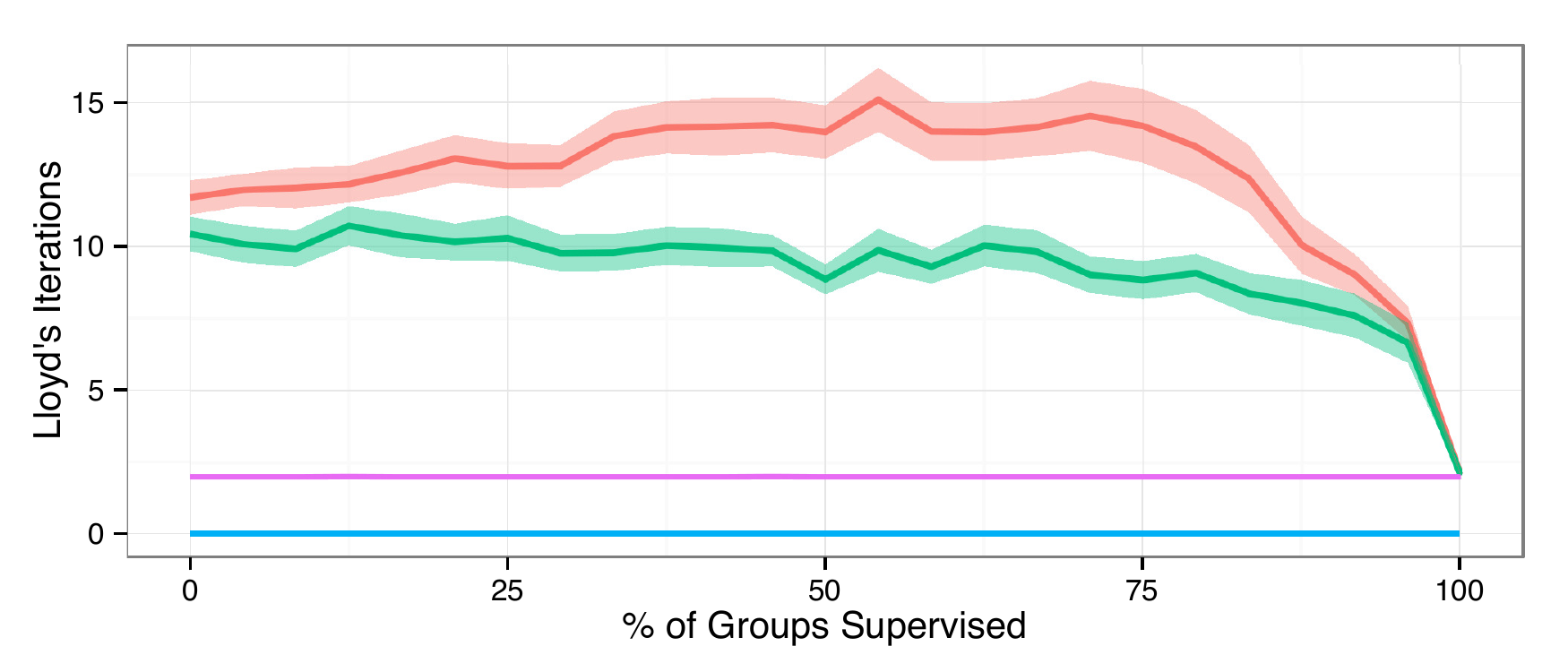}}
     \subfigure[{\tt Iris}]{ \includegraphics[width=.45\textwidth]{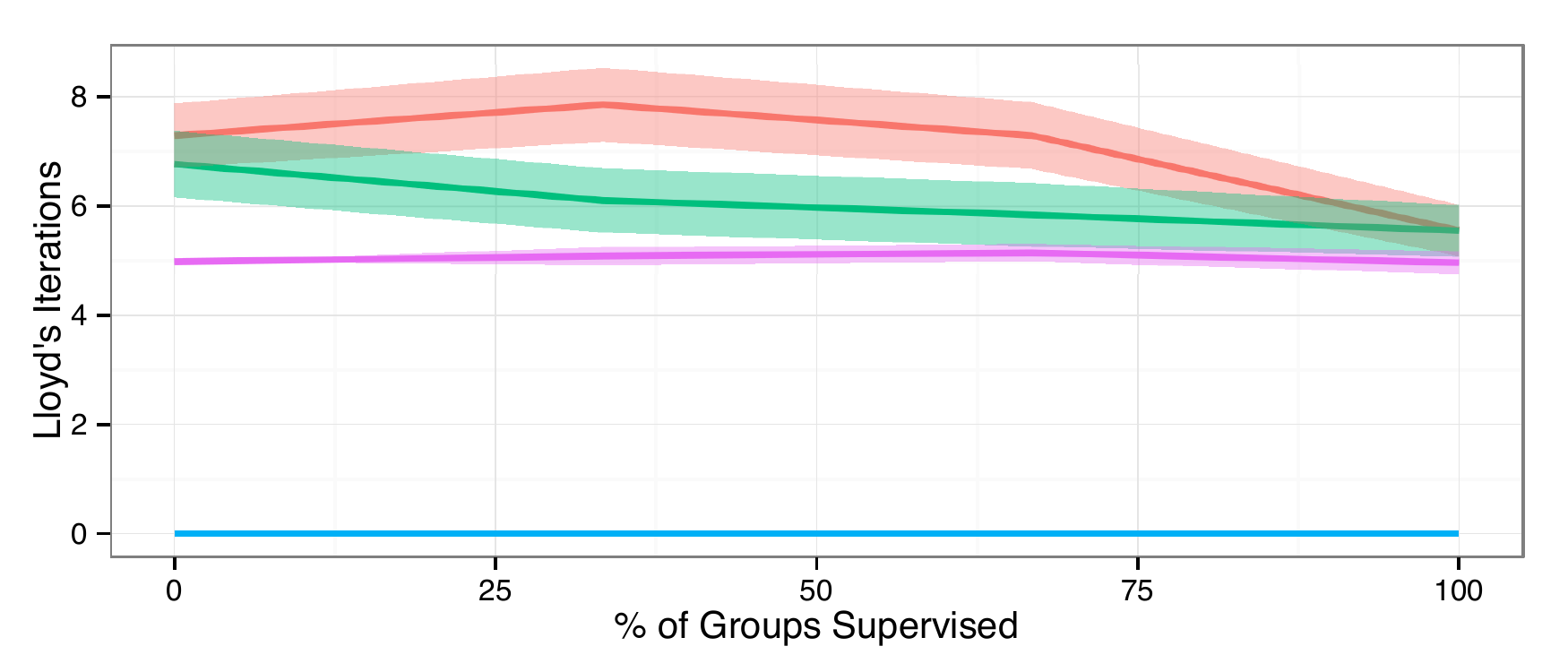}}
     \subfigure[{\tt Hyperspectral}]{ \includegraphics[width=.45\textwidth]{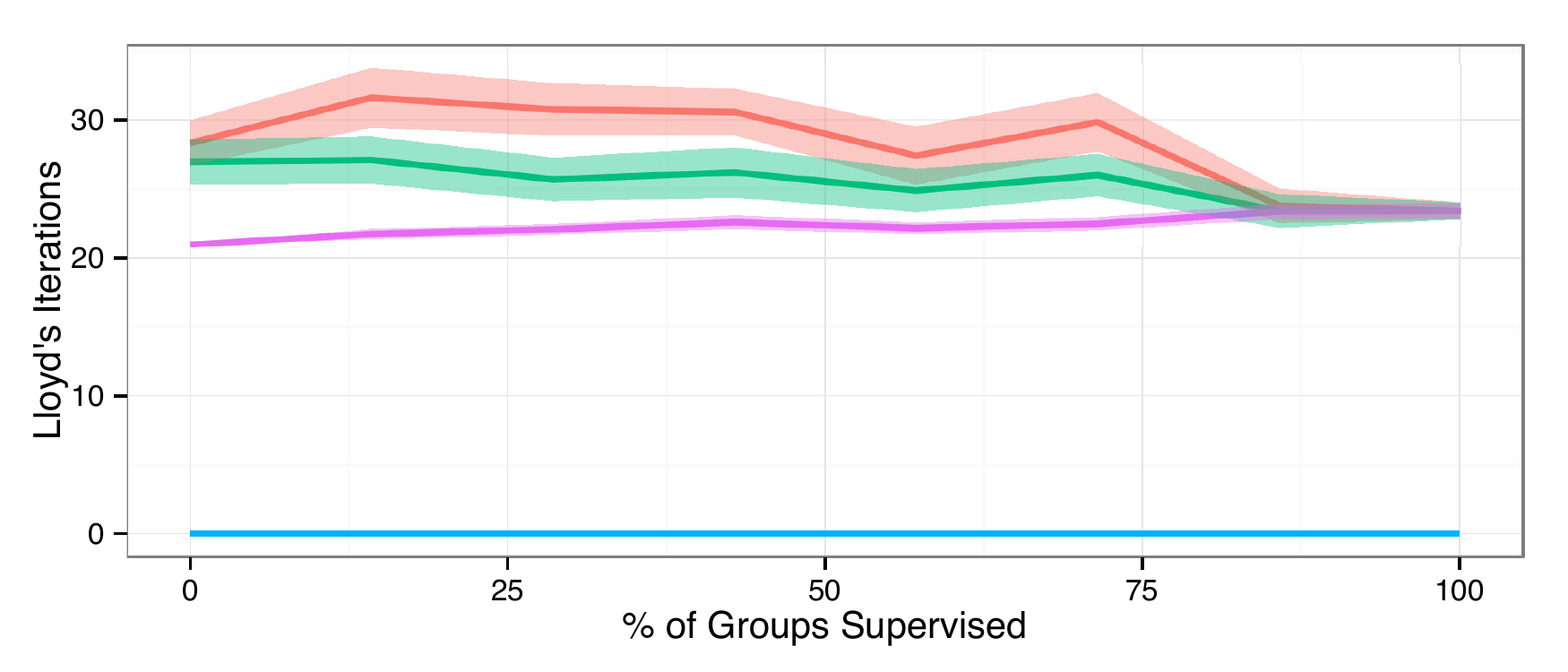}}
  \caption{Lloyd's iterations before convergence plotted as a function of the level of supervision for 100 Monte Carlo replicates.  Shading indicates $\pm$ two standard deviations.   Colors indicate algorithm:\\
    \textcolor[rgb]{.62, .65, 	.03}{gold}: {\tt Constrained-KMeans} (without Lloyds iterations);\\
    \textcolor[rgb]{0.0, .72, .96}{blue}:  {\tt ss-k-means++} (without Lloyds iterations);\\
    \textcolor[rgb]{1.0, .44, .42}{red}:  {\tt Constrained-KMeans};\\
    \textcolor[rgb]{0.0, .76, .53}{green}: {\tt ss-k-means++}; and\\
    \textcolor[rgb]{.91, .40, .94}{pink}: {\tt Constrained-KMeans} initialized at true centroids of labels.}
    \label{fig:iter}
\end{figure*}

Figure \ref{fig:ari} shows the ARI for all algorithms.  Note that supervision improves the ARI, as expected.  Also,  {\tt ss-k-means++} generally outperforms {\tt Constrained-KMeans}.  The same observation holds for the initialization only versions as well.  Remarkably, the true centroids and Lloyd's algorithm is outperformed by the initialization only methods on the {\tt Iris} and {\tt Hyperspectral} datasets at 100\% supervision for the ARI metric.  This is due to the fact that the true classes do not correspond to the minimum cost solution, which is what Lloyd's iterations would improve (apparently at the cost of ARI).

\begin{figure*}[t!]
    \centering
    \subfigure[{\tt Gaussian Mixture}] {\includegraphics[width=.45\textwidth]{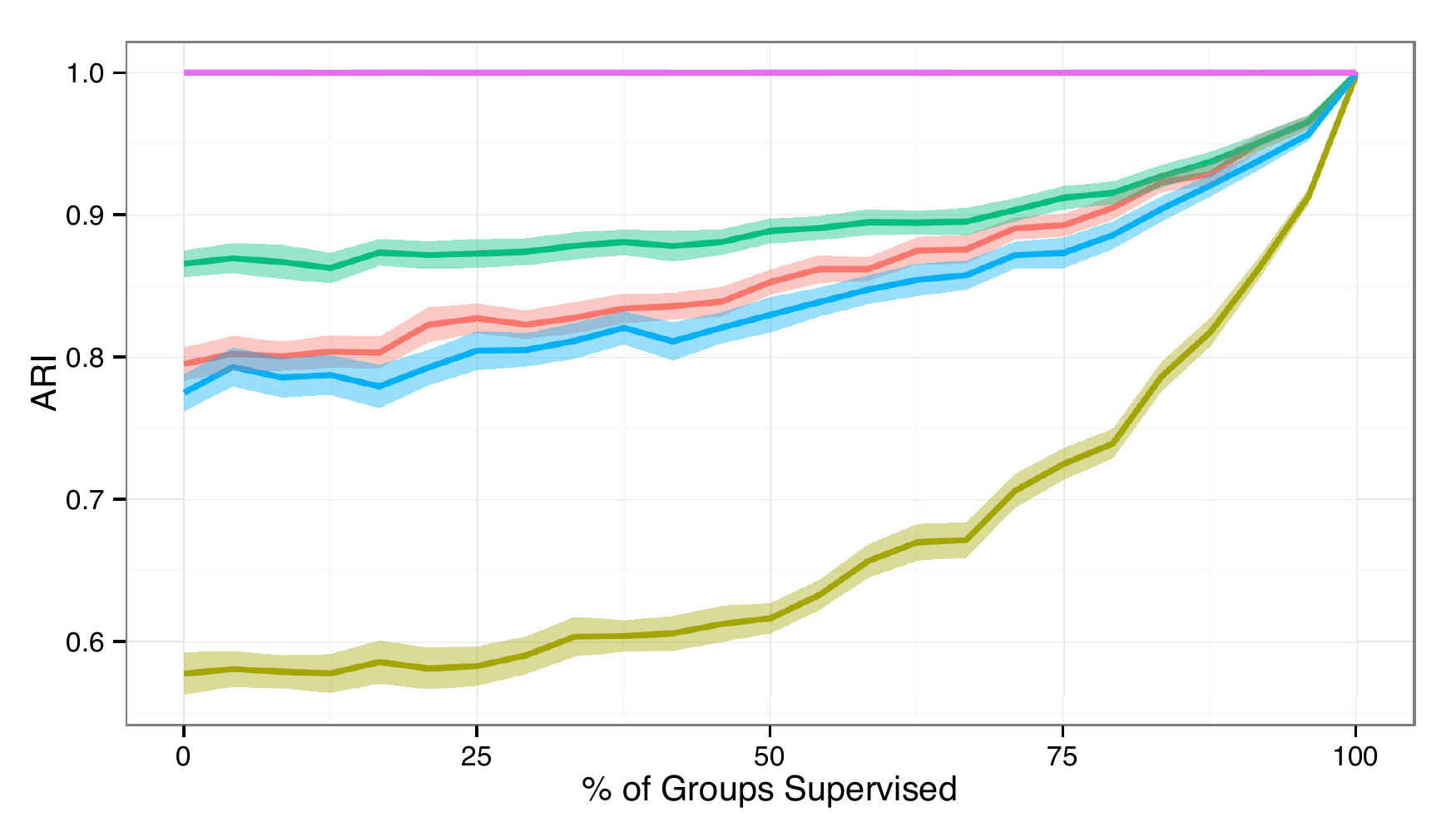}}
     \subfigure[{\tt Iris}]{ \includegraphics[width=.45\textwidth]{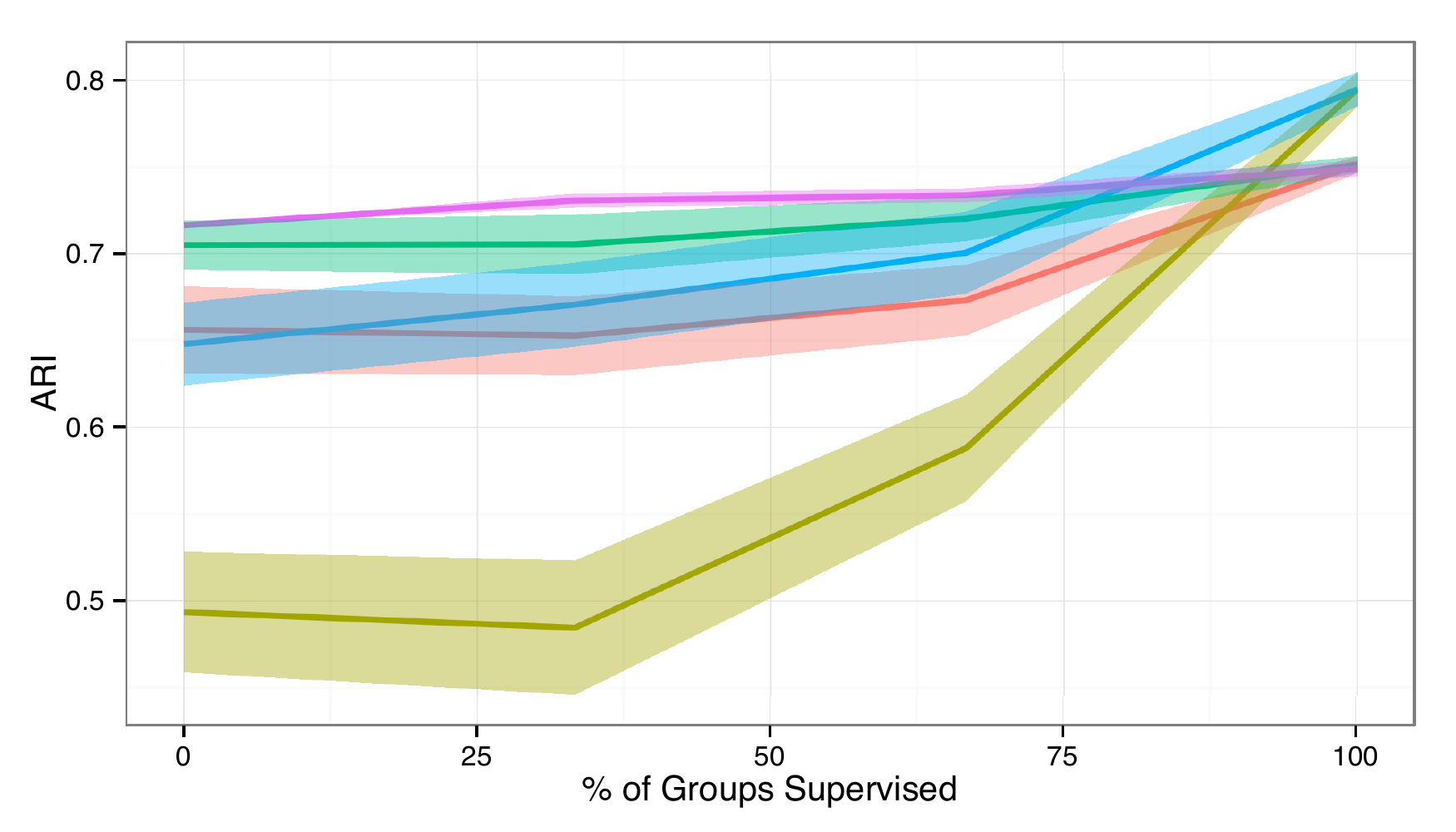}}
        \subfigure[{\tt Hyperspectral}]{ \includegraphics[width=.45\textwidth]{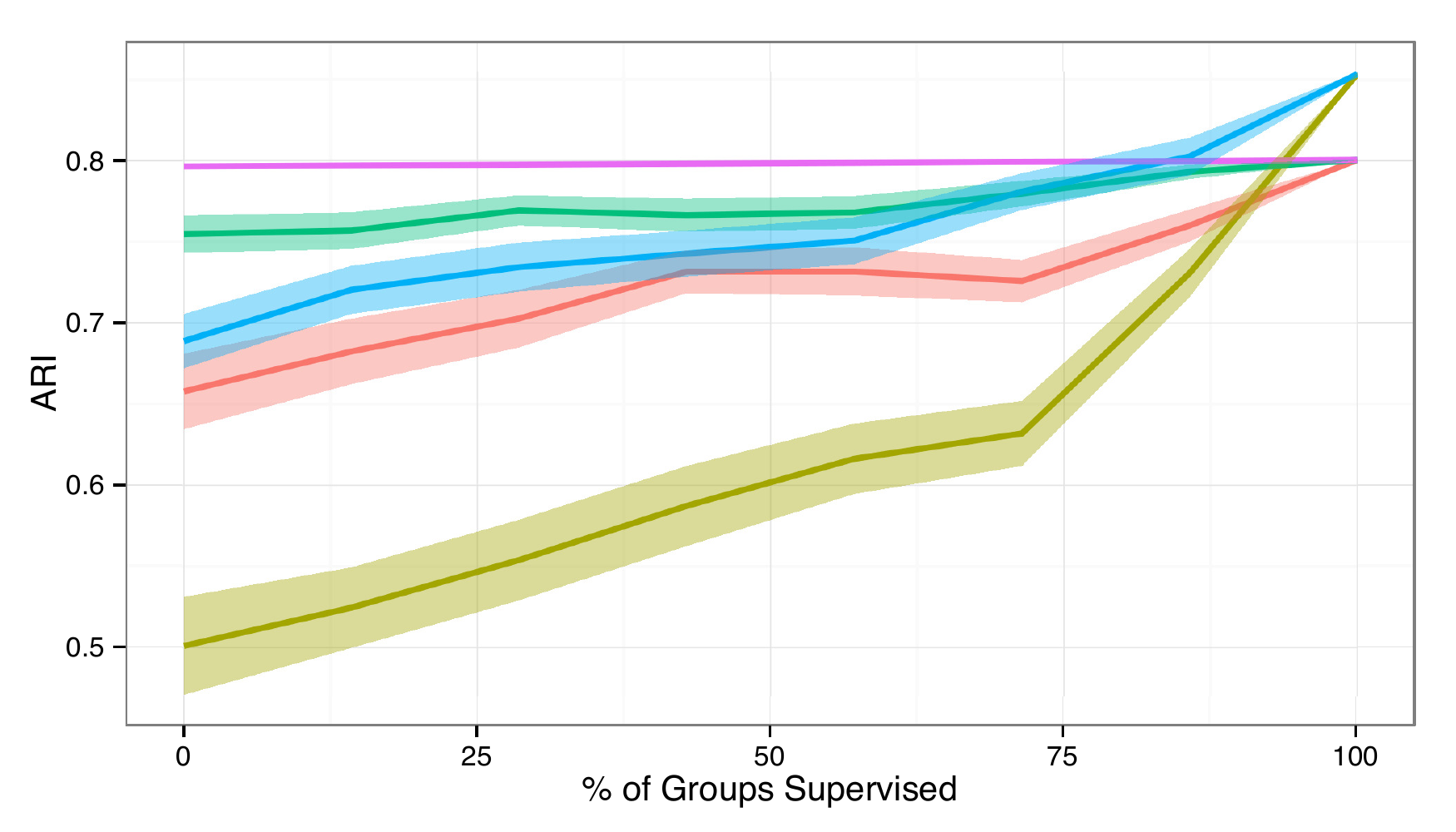}}
    \caption{Average ARI shown as a function of the level of supervision for 100 Monte Carlo replicates.  Shading indicates $\pm$ two standard deviations.  Green beats red.
    Colors indicate algorithm:\\
    \textcolor[rgb]{.62, .65, 	.03}{gold}: {\tt Constrained-KMeans} (without Lloyds iterations);\\
    \textcolor[rgb]{0.0, .72, .96}{blue}:  {\tt ss-k-means++} (without Lloyds iterations);\\
    \textcolor[rgb]{1.0, .44, .42}{red}:  {\tt Constrained-KMeans};\\
    \textcolor[rgb]{0.0, .76, .53}{green}: {\tt ss-k-means++}; and\\
    \textcolor[rgb]{.91, .40, .94}{pink}: {\tt Constrained-KMeans} initialized at true centroids of labels.}
    \label{fig:ari}
\end{figure*}

\section{Conclusions}
In this paper, we present a natural extension of {\tt k-means++} and {\tt Constrained-KMeans}.  Then, we prove the corresponding bound on the expectation of the cost under some conditions on the supervision.  No assumptions are made about the distribution of the data.  Finally, we demonstrated that on three datasets judicious supervision and good starting center selection heuristics improve clustering performance, cost, and iteration count.

Possible future theoretical work includes incorporating the advances set forth in the extensions to the original {\tt k-means++} paper.  For example, we could produce semi-supervised versions of {\tt k-means\#} \cite{bahmani2012scalable} and {\tt k-means||} \cite{ailon2009streaming} with commensurately improved bounds.  Relaxing the constraints to the pairwise cannot-link and must-link constraints as in  \cite{wagstaff2001constrained} is also desirable, because the assumption of exogenously provided hard labels is often untenable.  Other assumptions that would be nice to relax would be the equal cluster shapes and cluster volume implicit in {\tt k-means} clustering.

\section*{Acknowledgements}
The authors would like to thank Theodore D. Drivas for helping to test the codes used in the experiments and for consultation on aesthetics.

\section*{Funding}
This work is partially funded by the National Security Science and Engineering Faculty Fellowship (NSSEFF),
the Johns Hopkins University Human Language Technology Center of Excellence (JHU HLT COE), and the
XDATA program of the Defense Advanced Research Projects Agency (DARPA) administered through Air
Force Research Laboratory contract FA8750-12-2-0303.

\bibliographystyle{gSCS}

\end{document}